\documentclass{article} 
\usepackage{times}
\usepackage{algorithm}
\usepackage{algorithmic}

\usepackage{dsfont}
\usepackage{amsmath}
\usepackage{amsfonts}
\usepackage{amssymb}
\usepackage{amsthm}
\usepackage{subfigure}
\usepackage{adjustbox}
\usepackage{epsfig}
\usepackage{color}
\usepackage{enumerate}
\usepackage{placeins}

\newtheorem{lemma}{Lemma}
\newtheorem{theorem}{Theorem}

\newtheorem{definition}{Definition}
\newtheorem{assumption}{Assumption}

\theoremstyle{definition}

\renewcommand{\P}{\mathbb{P}}

\newcommand{\X}{\mathcal{X}}

\usepackage{amsmath,amsfonts,bm}

\def\eqref#1{equation~\ref{#1}}

\def\1{\bm{1}}

\DeclareMathAlphabet{\mathsfit}{\encodingdefault}{\sfdefault}{m}{sl}
\SetMathAlphabet{\mathsfit}{bold}{\encodingdefault}{\sfdefault}{bx}{n}

\newcommand{\R}{\mathbb{R}}

\usepackage{hyperref}
\usepackage{url}

\usepackage[accepted]{icml2021}

\icmltitlerunning{Locally Adaptive Label Smoothing for Predictive Churn}

\begin{document}

\twocolumn[
\icmltitle{Locally Adaptive Label Smoothing for Predictive Churn}



\begin{icmlauthorlist}
\icmlauthor{Dara Bahri}{google}
\icmlauthor{Heinrich Jiang}{google}

\end{icmlauthorlist}

\icmlaffiliation{google}{Google Research, Mountain View, USA}

\icmlcorrespondingauthor{Dara Bahri}{dbahri@google.com}

\icmlkeywords{Machine Learning, ICML}

\vskip 0.3in
]



\printAffiliationsAndNotice{}  

\newcommand{\fix}{\marginpar{FIX}}
\newcommand{\new}{\marginpar{NEW}}

\newcommand{\todo}[1]{{\color{red}\textbf{TODO}: #1}}

\begin{abstract}

Training modern neural networks is an inherently noisy process that can lead to high \emph{prediction churn}-- disagreements between re-trainings of the same model due to factors such as randomization in the parameter initialization and mini-batches-- even when the trained models all attain similar accuracies. Such prediction churn can be very undesirable in practice. In this paper, we present several baselines for reducing churn and show that training on soft labels obtained by adaptively smoothing each example's label based on the example's neighboring labels often outperforms the baselines on churn while improving accuracy on a variety of benchmark classification tasks and model architectures.

\end{abstract}

\section{Introduction}

Deep neural networks (DNNs) have proved to be immensely successful at solving complex classification tasks across a range of problems. Much of the effort has been spent towards improving their predictive performance (i.e. accuracy), while comparatively little has been done towards improving the \emph{stability} of training these models \cite{zheng2016improving}. Modern DNN training is inherently noisy due to factors such as the random initialization of network parameters \cite{glorot2010understanding}, the mini-batch ordering \cite{loshchilov2015online}, the effects of various data augmentation \cite{shorten2019survey} or pre-processing tricks \cite{santurkar2018does}, and the non-determinism arising from the hardware \cite{turner2015estimating},
all of which are exacerbated by the non-convexity of the loss surface \cite{scardapane2017randomness}. This results in local optima corresponding to models that have very different predictions on the same data points. This may seem counter-intuitive, but even when the different runs all produce very high accuracies for the classification task, their predictions can still differ quite drastically as we will show later in the experiments. Thus, even an optimized training procedure can lead to high {\it prediction churn}, which refers to the proportion of sample-level disagreements between classifiers caused by different runs of the same training procedure\footnote{Concretely, given two classifiers applied to the same test samples, the prediction churn between them is the fraction of test samples with different predicted labels.}.

In practice, reducing such predictive churn can be critical. For example, in a production system, models are often continuously improved on by being trained or retrained with new data or better model architectures and training procedures. In such scenarios, a candidate model for release must be compared to the current model serving in production. Oftentimes, this decision is conditioned on more than just overall offline test accuracy-- in fact, the offline metrics are often not completely aligned with the actual goal, especially if these models are used as part of a larger system (e.g. maximizing offline click-through rate vs. maximizing revenue or user satisfaction) \cite{deng2013improving,beel2013comparative,dmitriev2016measuring}. As a result, these comparisons require extensive and costly live experiments, requiring human evaluation in situations where the candidate and the production model disagree (i.e. in many situations, the true labels are not available without a manual labeler) \cite{theocharous2015ad,deng2015objective,deng2016data}. In these cases, it can be highly desirable to lower predictive churn.

Despite the practical relevance of lowering churn, there has been surprisingly little work done in this area, which we highlight in the related work section. In this work, we focus on predictive churn reduction under retraining the same model architecture on an identical train and test set. Our main contributions are as follows:
\begin{itemize}
    \item We provide one of the first comprehensive analyses of baselines to lower prediction churn, showing that popular approaches designed for other goals are effective baselines for churn reduction, even compared to methods designed for this goal.
    \item We improve label smoothing, a {\it global} smoothing method popular for calibrating model confidence, by utilizing the {\it local} information leveraged by the $k$-NN labels thus introducing a locally adaptive label smoothing which we show to often outperform the baselines on a wide range of benchmark datasets and model architectures.
    \item We show new theoretical results for the $k$-NN labels suggesting the usefulness of the $k$-NN label. We show under mild nonparametric assumptions that for a wide range of $k$, the $k$-NN labels uniformly approximates the optimal soft label and when $k$ is tuned optimally, achieves the minimax optimal rate. We also show that when $k$ is linear in $n$, the distribution implied by the $k$-NN label approximates the original distribution smoothed with an {\it adaptive} kernel.
\end{itemize}

\section{Related Works}
Our work spans multiple sub-areas of machine learning. The main problem this paper tackles is reducing prediction churn. In the process, we show that label smoothing is an effective baseline and we improve upon it in a principled manner using deep $k$-NN label smoothing to obtain a locally adaptive version of it.

{\bf Prediction Churn.} There are only a few works which explicitly address prediction churn. \citet{fard2016launch} proposed training a model so that it has small prediction instability with future versions of the model by modifying the data that the future versions are trained on. They furthermore propose turning the classification problem into a regression towards corrected predictions of an older model as well as regularizing the new model towards the older model using example weights. \citet{cotter2019optimization,goh2016satisfying} use constrained optimization to directly lower prediction churn across model versions.
Simultaneously training multiple identical models (apart from initialization) while tethering their predictions together via regularization has been proposed in the context of distillation~\citep{anil2018large,zhang2018deep,zhu2018knowledge,song2018collaborative} and robustness to label noise~\citep{malach2017decoupling,han2018co}. This family of methods was termed ``co-distillation'' by \citet{anil2018large}, who also noted that it can be used to reduce churn in addition to improving accuracy. In this paper, we show much more extensively that co-distillation is indeed a reasonable baseline for churn reduction.


{\bf Label Smoothing.} Label smoothing \citep{szegedy2016rethinking} is a simple technique wherein the model is trained on the soft labels obtained by a convex combination of the hard true label and the soft uniform distribution across all the labels. It has been shown that it leads to better confidence calibration and generalization~\citep{muller2019does}. Here we show that label smoothing is a reasonable baseline for reducing prediction churn, and we moreover enhance it for this task by smoothing the labels {\it locally} via $k$-NN rather than a pure {\it global} approach mixing with the uniform distribution.

{\bf $k$-NN Theory.} The theory of $k$-NN classification has a long history (e.g. \citet{fix1951discriminatory,cover1968rates,stone1977consistent,devroye1994strong,chaudhuri2014rates}). To our knowledge, the most relevant $k$-NN classification result is by \citet{chaudhuri2014rates}, who show statistical risk bounds under similar assumptions as used in our work. Our analysis shows finite-sample $L_\infty$ bounds on the $k$-NN labels, which is a stronger notion of consistency as it provides a uniform guarantee, rather than an {\it average} guarantee as is shown in previous works under standard risk measures such as $L_2$ error. We do this by leveraging recent techniques developed in \citet{jiang2019non} for $k$-NN regression, which assumes an additive noise model instead of classification. Moreover, we provide to our knowledge the first consistency guarantee for the case where $k$ grows linearly with $n$.

{\bf Deep $k$-NN.} $k$-NN is a classical method in machine learning which has recently been shown to be useful when applied to the intermediate embeddings of a deep neural network \citep{papernot2018deep} to obtain more calibrated and adversarially robust networks. This is because standard distance measures are often better behaved in these representations leading to better performance of $k$-NN on these embeddings than on the raw inputs. \citet{TrustScores} uses nearest neighbors on the intermediate representations to obtain better uncertainty scores than softmax probabilities and \citet{bahri2020deep} uses the $k$-NN label disagreement to filter noisy labels for better training. Like these works, we also leverage $k$-NN on the intermediate representations but we show that utilizing the $k$-NN labels leads to lower prediction churn. 
\section{Algorithm}

\begin{figure*}
    \centering
    \includegraphics[width=\textwidth]{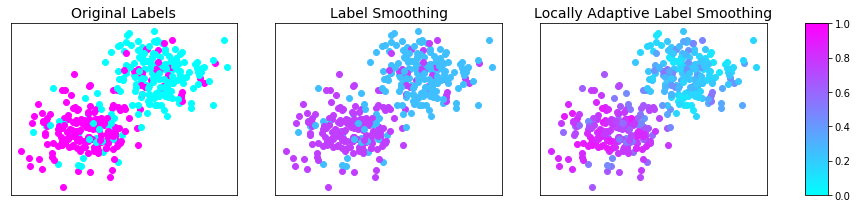}
    \caption{{\bf Visualization of the effects of global vs locally adaptive label smoothing}. This visualization provides intuition for why our locally adaptive label smoothing method can improve neural network training stability. {\bf Left}: A binary classification dataset in $2$ dimensions where there are magenta points for the positive class and cyan points for the negative class. The data is generated from a mixture of two Gaussians, where the bottom-left Gaussian corresponds to the positive examples and the top-right corresponds to the negative examples, and to add label noise, we swap the labels of $10\%$ chosen uniformly. {\bf Middle:} We see that label smoothing simply pushes the labels uniformly closer to the average label ($0.5$). In particular, we see that the noisy labels still remain and thus may still cause conflicting information during training, possibly leading to predictive churn. {\bf Right}: We now show our locally adaptive label smoothing approach, which also smooths the labels based on local information. This alleviates the examples with noisy labels by bringing them more in line with the average label amongst its neighbors and provides a more locally smooth label profile with respect to the input space. Such smoothness can help model training converge in a more stable manner.}
    \label{fig:visualization}
\end{figure*}

Suppose that the task is multi-class classification with $L$ classes and the training datapoints are $(x_1,y_1),...,(x_n,y_n)$, where $x_i \in \mathcal{X}$, and $\mathcal{X}$ is a compact subset of $\mathbb{R}^D$ and $y_i \in \mathbb{R}^L$, represents the one-hot vector encoding of the label -- that is, if the $i$-th example has label $j$, then $y_i$ has $1$ in the $j$-th entry and $0$ everywhere else. We give the formal definition of the smoothed labels:
\begin{definition}[Label Smoothing] Given label smoothing parameter $0 \le a \le 1$, then the smoothed label $y$ is (where ${\bf 1}_L$ denotes the vector of all $1$'s in $\mathbb{R}^L$).
\begin{align*}
    y^{LS}_a := (1 - a) \cdot y + \frac{a}{L} \cdot {\bf 1}_{L}.
\end{align*}
\end{definition}

We next formally define the $k$-NN label, which is the average label of the example's $k$-nearest neighbors in the training set.
Let us use shorthand $X := \{x_1,...,x_n\}$ and $y_i \in \mathbb{R}^L$.
\begin{definition} [$k$-NN label]
Let the $k$-NN radius of $x \in \mathcal{X}$ be $r_k(x) := \inf \{ r : |B(x, r) \cap X| \ge k \}$ where $B(x, r) := \{x' \in \mathcal{X} : |x - x'| \le r \}$ and the $k$-NN set of $x \in \mathcal{X}$ be $N_k(x) := B(x, r_k(x)) \cap X$. 
Then for all $x \in \mathcal{X}$, 
the $k$-NN label is defined as
\begin{align*}
    \eta_k(x) := \frac{1}{|N_k(x)|} \sum_{i=1}^n  y_i\cdot 1\left[ x_i \in N_k(x) \right].
\end{align*}
\end{definition}

The label smoothing method can be seen as performing a global smoothing. That is, every label is equally transformed towards the uniform distribution over all labels. While it seems almost deceptively simple, it has only recently been shown to be effective in practice, specifically for better calibrated networks \cite{muller2019does}. However, since this smoothing technique is applied equally to all datapoints, it fails to incorporate local information about the datapoint. To this end, we propose using the $k$-NN label, which smooths the label across its nearest neighbors. We show theoretically that the $k$-NN label can be a strong proxy for the optimal soft label, that is, the expected label given the features and thus the best prediction one can make given the uncertainty under an $L_2$ risk measure. In other words, compared to the true label (or even the label smoothing), the $k$-NN label is robust to variability in the data distribution and provides a more stable estimate of the label than the original hard label which may be noisy. Training on such noisy labels have been shown to hurt model performance \citep{bahri2020deep} and using the smoothed labels can help mitigate these effects. To this end, we define $k$-NN label smoothing as follows:
\begin{definition}[$k$-NN label smoothing]\label{def:knn_label_smoothing} Let $0 \le a, b \le 1$ be $k$-NN label smoothing parameters. Then the $k$-NN smoothed label of datapoint $(x, y)$ is defined as:
\begin{align*}
    y^{\text{kNN}}_{a, b} = (1 - a) \cdot y + a \cdot \left(b\cdot \frac{1}{L}\cdot {\bf 1}_L + (1 - b)\cdot \eta_k(x)\right).
\end{align*}
\end{definition}
We see that $a$ is used to weight between using the true labels vs. using smoothing, and $b$ is used to weight between the global vs. local smoothing. We provide an illustrative simulation in Figure~\ref{fig:visualization}. Algorithm~\ref{alg:deepknn} shows how $k$-NN label smoothing is applied to deep learning models. Like \citet{bahri2020deep}, we perform $k$-NN on the network's logits layer.

\begin{algorithm}[h]
\caption{Deep $k$-NN locally adaptive label smoothing}
\label{alg:deepknn}
\begin{algorithmic}[1]
   \STATE \textbf{Inputs:} $0 \le a, b \le 1$, $k$, training data $(x_1,y_1),...,(x_n,y_n)$, model training procedure $\mathcal{M}$.
   \STATE Train model $M_0$ on $(x_1,y_1),...,(x_n,y_n)$ with $\mathcal{M}$.
   \STATE Let $z_1,...,z_n \in \mathbb{R}^L$ be the logits of $x_1,...,x_n$, respectively, w.r.t. $M_0$.
   \STATE Let $y^{\text{kNN}}_i$ be the $k$-NN smoothed label (see Definition~\ref{def:knn_label_smoothing}) of $(z_i, y_i)$ computed w.r.t. $(z_1,y_1),...,(z_n,y_n)$.
   \STATE Train model $M$ on $(x_1,y^\text{kNN}_1),...,(x_n,y^\text{kNN}_n)$ with $\mathcal{M}$. 
\end{algorithmic}
\end{algorithm}

\section{Theoretical Analysis}

In this section, we provide theoretical justification for why the $k$-NN labels may be useful. 
In particular, we show results for two settings, where $n$ is the number of datapoints.
\begin{itemize}
    \item When $k \ll n$, we show that with appropriate setting of $k$, the $k$-NN smoothed labels approximate the predictions of the optimal soft classifier at a minimax-optimal rate.
    \item When $k = O(n)$, we show that the distribution implied by the $k$-NN smoothed labels is equivalent to the original distribution convolved with an {\it adaptive} smoothing kernel.
\end{itemize}

Our results may also reveal insights into why distillation methods (the procedure of training a model on another model's predictions instead of the true labels) can work. Another way of considering the result is that the $k$-NN smoothed label is equivalent to the soft prediction of the $k$-NN classifier. Thus, if one were to train on the $k$-NN labels, it would essentially be distillation on the $k$-NN classifier and our theoretical results show that the labels implied by $k$-NN approximate the predictions of the {\it optimal} classifier (in the $k \ll n$ setting). Learning the optimal classifier may indeed be a better goal than learning from the true labels, because the latter may lead to overfitting to the sampling noise rather than just the true signal implied by the optimal classifer.
While distillation is not the topic of this work, our results in this section may be of independent interest to that area.

For the analysis, we assume the binary classification setting, but it is understood that our results can be straightforwardly generalized to the multi-class setting. The feature vectors are defined on compact support $\X \subseteq \R^D$ and datapoints are drawn as follows: the feature vectors are drawn from density $p_\X$ on $\mathcal{X}$ and the labels are drawn according to the label function $\eta : \X \rightarrow [0, 1]$, i.e. $\eta(x) = \P(Y = 1|X=x)$.

\subsection{$k\ll n$}

We make a few mild regularity assumptions for our analysis to hold, which are standard in works analyzing non-parametric methods \cite{singh2009adaptive,chaudhuri2014rates,reeve2019fast,jiang2019non,bahri2020deep}.
The first part ensures that the support $\mathcal{X}$ does not become arbitrarily thin anywhere, the second ensures that the density does not vanish anywhere in the support, and the third ensures that the label function $\eta$ is smooth w.r.t. to its input.
\begin{assumption} \label{assumption}The following three conditions hold:
\begin{itemize}
\item Support Regularity: There exists $\omega > 0$ and $r_0 > 0$ such that $\text{Vol}(\mathcal{X} \cap B(x, r)) \ge \omega \cdot \text{Vol}(B(x, r))$ for all $x \in \mathcal{X}$ and $0 < r < r_0$, where $B(x, r) := \{x' \in \X : |x - x'| \le r\}$.
\item Non-vanishing density: $p_{X, 0} := \inf_{x \in \mathcal{X}} p_X(x) > 0$.
\item Smoothness of $\eta$: There exists $0 < \alpha \le 1$ and $C_\alpha > 0$ such that $|\eta(x) - \eta(x')| \le C_\alpha |x - x'|^\alpha$ for all $x, x' \in \mathcal{X}$. 
\end{itemize}
\end{assumption}

We have the following result which provides a {\it uniform} bound between the smoothed $k$-NN label $\eta_k$ and the optimal soft label $\eta$.
\begin{theorem}\label{theo:knn_bound} Let $0 < \delta < 1$ and suppose that Assumption~\ref{assumption} holds and that $k$ satisfies the following:
\begin{align*}
2^8 \cdot D \log^2(4/\delta) \cdot \log n \le k \le \frac{1}{2} \cdot \omega  \cdot p_{X,0}\cdot v_D \cdot r_0^D \cdot n,
\end{align*}
where $v_D := \frac{\pi^{D/2}}{\Gamma(d/2+1)}$ is the volume of a $D$-dimensional unit ball. Then with probability at least $1- \delta$, we have
\begin{align*}
    \sup_{x\in \mathcal{X}} |\eta_k(x) - \eta(x)| \le &C_\alpha \left( \frac{2k}{\omega \cdot v_D \cdot n \cdot p_{X, 0}}\right)^{\alpha/D} \\
    &+ \sqrt{\frac{2\log(4D/\delta) + 2D\log(n)}{k}}.
\end{align*}
\end{theorem}

In other words, there exists constants $C_1, C_2, C$ depending on $\eta$ and $\delta$ such that if $k$ satisfies
\begin{align*}
    C_1 \log n \le k \le C_2 \cdot n,
\end{align*}
then with probability at least $1-\delta$, ignoring logarithmic factors in $n$ and $1/\delta$:
\begin{align*}
    \sup_{x\in \mathcal{X}} |\eta_k(x) - \eta(x)| \le C\cdot \left(\left(\frac{k}{n}\right)^{\alpha/D} + \frac{1}{\sqrt{k}}\right).
\end{align*}
Choosing $k \approx n^{2\alpha/(2\alpha +D)}$, gives us a bound of $\sup_{x\in \mathcal{X}} |\eta_k(x) - \eta(x)| \le \widetilde{O}(n^{-1/(2\alpha+D)})$, which is the minimax optimal rate as established by \citet{tsybakov1997nonparametric}.

Therefore, the advantage of using the smoothed labels $\eta_k(x_1),...,\eta_k(x_n)$ instead of the true labels $y_1,...,y_n$,  is that the smoothed labels approximate the optimal soft classifier. Moreover, as shown above, with appropriate setting of $k$, the smoothed labels are a minimax-optimal estimator of the true label function $\eta$. Thus, the smoothed labels provide as good of a proxy for $\eta$ as any estimator possibly can. 

As suggested earlier, another way of considering this result is that the original labels may contain considerable noise and thus no single label can be guaranteed reliable. Using the smoothed label instead mitigates this effect and allows us to train the model to match the label function $\eta$.

\subsection{$k$ linear in $n$}

In the previous subsection, we showed the utility of $k$-NN label smoothing as a theoretically sound proxy for the optimal soft labels, which attains statistical consistency guarantees as long as $k$ grows faster than $\log n$ and $k/n\rightarrow 0$. Now, we analyze the case where $k$ grows linearly with $n$. In this case, the $k$-NN smoothed labels no longer recover the optimal soft label function $\eta$, but instead an adaptive kernel smoothed version of $\eta$. We make this relationship precise here.

Suppose that $k = \lfloor \beta\cdot n \rfloor$ for some $0 < \beta < 1$. We define the $\beta$-smoothed label function.
\begin{definition}[$\beta$-smoothed label function]
Let $r_\beta(x) := \inf\{r > 0 : \mathcal{P}(B(x, r)) \ge \beta \}$, that is the radii of the smallest ball centered at $x$ with probability mass $\beta$ w.r.t. $P_X$. Then, let $\widetilde{\eta}_\beta(x)$ be the expectation of $\eta$ on $B(x, r_\beta(x))$ w.r.t. $P_X$:
\begin{align*}
    \widetilde{\eta}_\beta(x) := \frac{1}{\beta} \int_{B(x, r_\beta(x))} \eta(x) \cdot P_X(x) dx.
\end{align*}
\end{definition}
We can view $\widetilde{\eta}_\beta$ as an adaptively kernel smoothed version of $\eta$, where adaptivity arises from the density of the point (the more dense, the smaller the bandwidth we smooth it across) and the kernel is based on the density.

We now prove the following result which shows that in this setting $\eta_k$ estimates $\widetilde{\eta}_\beta(x)$. It is worth noting that we need very little assumption on $\eta$ as compared to the previous result because the $\beta$-smoothing of $\eta$ provides a more regular label function; moreover, the rates are fast (i.e. $\widetilde{O}(\sqrt{D/n})$).

\begin{theorem}\label{theo:knn_linear}
Let $0 < \delta < 1$ and $k = \lfloor \beta\cdot n \rfloor$. Then with probability at least $1- \delta$, we have for $n$ sufficiently large depending on $\beta, \delta$:
\begin{align*}
    \sup_{x \in \mathcal{X}} |\eta_k(x) - \widetilde{\eta}_\beta(x)| \le 
   3 \sqrt{\frac{2\log(4D/\delta) + 2D\log(n)}{\beta\cdot n}}.
\end{align*}
\end{theorem}

\section{Experiments}

We now describe the experimental methodology and results for validating our proposed method.
\subsection{Baselines}
We start by detailing the suite of baselines we compare against. We tune baseline hyper-parameters extensively, with the precise sweeps and setups available in the Appendix.
\begin{itemize}
    \item {\bf Control}: Baseline where we train for accuracy with no regards to churn.
    \item {\bf $\ell_p$ Regularization}: We control the stability of a model's predictions by simply regularizing them (independently of the ground truth label) using classical $\ell_p$ regularization. The loss function is given by:
    \begin{align*}
     \mathcal{L}_{\ell_p}(x_i, y_i) = \mathcal{L}(x_i, y_i) + a ||f(x_i)||_p.
     \end{align*}
     We experiment with both $\ell_1$ and $\ell_2$ regularization.
    \item {\bf Bi-tempered}: This is a baseline by \citet{amid2019robust}, originally designed for robustness to label noise. It modifies the standard logistic loss function by introducing two temperature scaling parameters $t_1$ and $t_2$. We apply their ``bi-tempered'' loss here, suspecting that methods which make model training more robust to noisy labels may also be effective at reducing churn.
    \item {\bf Anchor}: This is based on a method proposed by \citet{fard2016launch} specifically for churn reduction. It uses the predicted probabilities from a preliminary model to smooth the training labels of the second model. We first train a preliminary model $f_\text{prelim}$ using regular cross-entropy loss. We then retrain the model using smoothed labels $(1-a)y_i + af_\text{prelim}(x_i)$, thus  ``anchoring'' on a preliminary model's predictions. In our experiments, we train one preliminary model and fix it across the runs for this baseline to reduce prediction churn.
    \item {\bf Co-distillation}: We use the co-distillation approach presented by \citet{anil2018large}, who touched upon its utility for churn reduction. We train two identical models $M_1$ and $M_2$ (but subject to different random initialization) in tandem while penalizing divergence between their predictions. The overall loss is
\begin{align*}
\mathcal{L}_\text{codistill}(x_i, y_i) =\mathcal{L}(f_1(x_i), y_i) &+ \mathcal{L}(f_2(x_i), y_i) \\
&+ a \Psi(f_1(x_i), f_2(x_i)).
\end{align*}
In their paper the authors set $\Psi$ to be cross-entropy:
\begin{align*}
\Psi(p^{(1)}, p^{(2)}) = \sum_{i\in[L]}p^{(1)}_i\log(p^{(2)}_i),
\end{align*}
but they note KL divergence can be used. We experiment with both cross-entropy and KL divergence. We also tune $n_\text{warm}$, the number of burn-in steps of training before turning on the regularizer.
    \item {\bf Label Smoothing}: This is the method of \citet{szegedy2016rethinking} defined earlier in the paper.
Our proposed method augments global label smoothing by leveraging the local $k$-NN estimates. Naturally, we compare against doing global smoothing only and this serves as a key ablation model to see the added benefits of leveraging the $k$-NN labels.
\item {\bf Mixup}: This method proposed by \citet{zhang2017mixup} generates synthetic training examples on the fly by convex-combining random training inputs and their associated labels, where the combination weights are random draws from a $\text{Beta}(a,a)$ distribution. Mixup improves generalization, increases robustness to adversarial examples as well as label noise, and also improves model calibration~\citep{thulasidasan2019combating}.
\item {\bf {Ensemble}}: Ensembling deep neural networks can improve the quality of their uncertainty estimation \citep{lakshminarayanan2017simple,fort2019deep}. We consider the simple case where $m$ identical deep neural networks are trained independently on the same training data, and at inference time, their predictions are uniformly averaged together.
\end{itemize}

\subsection{Datasets and Models.}
We do not use standard data augmentation strategies for the image datasets so that the training data is constant across different training rounds. For all datasets we use the Adam optimizer with default learning rate $0.001$. We use a minibatch size of $128$ throughout.
\begin{itemize}
    \item {\bf MNIST}: We train a three-layer MLP with 256 hidden units and ReLU activations for 20 epochs.
    \item {\bf Fashion MNIST}: We use the same architecture as the one used for MNIST.
    \item {\bf SVHN}: We use LeNet5 CNN~\citep{lecun1998gradient} for $30$ epochs on the Google Street View Housing Numbers (SVHN) dataset, where each image is cropped to be $32\times32$ pixels.
    \item {\bf CelebA}: CelebA~\citep{liu2018large} is a large-scale face attributes dataset with more than $200$k celebrity images, each with 40 attribute annotations. We use the standard train and test splits, which consist of $162770$ and $19962$ images respectively. Images were resized to be $28\times28\times3$. We select the ``smiling'' and ``high cheekbone'' attributes and perform binary classification, training LeNet5 for $20$ epochs.
    \item {\bf Phishing}: To validate our method beyond the image classification setting, we train a three-layer MLP with $256$ hidden units per layer on UCI Phishing dataset~\citep{Dua:2019}, which consists of $7406$ train and $3649$ test examples on a $30$-dimensional input feature. 

\end{itemize}

\begin{figure}
    \centering
    \includegraphics[width=0.45\textwidth]{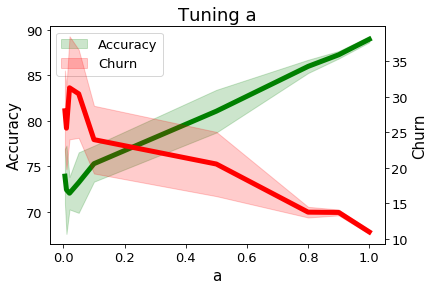}
    \includegraphics[width=0.47\textwidth]{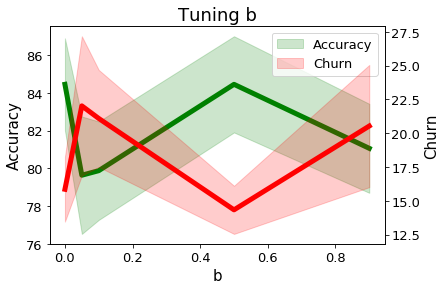}
    \includegraphics[width=0.5\textwidth]{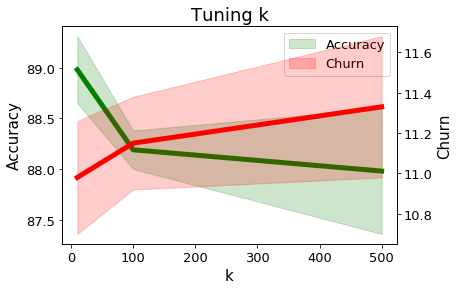}
    \caption{{\bf Performance across hyperparameters for SVHN}. For each of the three hyperparameters of our method ($a$, $b$, and $k$), we show the performance across a range of this hyperparameter keeping the other two hyperparameters fixed. Standard error bands are shaded. {\bf Top}: We tune $a$ while fixing $k = 10$ and $b = 0.9$. We see better performance under both accuracy and churn as $a$ increases, which suggests that the less weight we put on the original label, the better. {\bf Middle}: We tune $b$ while fixing $k=10$ and $a=0.5$. We don't see any clear pattern which suggests that $b$ is an essential hyperparameter trading off the locally adaptive vs global smoothing-- this suggests that our adding the locally adaptive component to the label smoothing is indeed having an effect on performance. {\bf Bottom}: We tune $k$ while fixing $a=1$ and $b=0.9$. We see that the accuracy and churn have little differences across a wide range of $k$ from $k=10$ to $k=500$, which suggests that $k$ is not an essential hyperparameter and that we are stable in it.}
    \label{fig:hyperparameter}
\end{figure}

\begin{table*}[]
\centering
\small
\begin{tabular}{l|l|l|l|l|l}
\hline
Dataset       & Method                                                                    & Accuracy \% & Churn \% & Churn Correct & Churn Incorrect  \\ \hline
              & \begin{tabular}[c]{@{}l@{}}$k$-NN LS (k=10, a=1, b=0.9)\end{tabular} & 88.98 (0.33)      & 10.98 (0.28)   & 4.64 (0.29)            & 62.23 (1.22)             \\
              & Label Smoothing (a=0.9)                                                      & 87.26 (0.73)      & 13.46 (0.62)   & 5.31 (0.57)            & 67.2 (1.44)              \\
              & Anchor (a=1.0)                                                 & 87.17 (0.16)      & 12.48 (0.39)   & 5.19 (0.2)             & 61.66 (1.85)             \\
SVHN          & $\ell_2$ Reg (a=0.5)                                                          & 88.16 (0.35)      & 11.85 (0.35)   & 5.07 (0.16)            & 62.73 (2.1)              \\
              & $\ell_1$ Reg (a=0.2)                                                          & 74.18 (3.41)      & 22.89 (3.74)   & 9.58 (4.04)            & 59.36 (5.7)              \\
              & Co-distill (CE, a=0.5)                                                       & 87.64 (0.64)      & 12.46 (0.48)   & 5.16 (0.51)            & 63.82 (1.67)             \\
              & Co-distill (KL, a=0.5)                                                       & 87.52 (0.45)      & 13.01 (0.3)    & 5.54 (0.33)            & 65.44 (1.46)             \\
              & Bi-tempered ($t_1$=0.5, $t_2$=1)                                  & 88.04 (0.5)       & 12.03 (0.3)    & 5.26 (0.3)             & 62.48 (1.83)             \\
              & Mixup (a=0.5)                                  & {\bf89.08 (0.18)}       & \bf{9.56 (0.16)}    & {\bf4.07 (0.15)}             & 54.75 (0.95)             \\
              & Control                                                                   & 86.64 (0.54)      & 14.64 (0.51)   & 6.03 (0.5)             & 69.59 (1.32)             \\ \hline
              & $k$-NN LS (k=5, a=0.9, b=0.9)                             & {\bf 98.23 (0.11)}      & {\bf 1.52 (0.12)}    & {\bf 0.7 (0.1)}              & 47.16 (3.39)             \\
              & Label Smoothing (a=0.9)                                                        & 98.15 (0.07)      & 1.65 (0.05)    & 0.71 (0.07)            & 50.73 (2.62)             \\
              & Anchor (a=1.0)                                                              & 97.72 (0.11)      & 2.66 (0.2)     & 1.21 (0.14)            & 64.51 (4.13)             \\
MNIST         & $\ell_2$ Reg (a=0.5)                                                              & 98.08 (0.1)       & 1.67 (0.12)    & 0.8 (0.08)             & 46.65 (3.2)              \\
              & $\ell_1$ Reg (a=0.01)                                                             & 97.67 (0.29)      & 2.51 (0.31)    & 1.3 (0.27)             & 56.8 (2.84)              \\
              & Co-distill (CE, a=0.2, $n_\text{warm}$=2k)                                          & 98.08 (0.06)      & 2.08 (0.11)    & 0.98 (0.07)            & 58.6 (3.91)              \\
              & Co-distill (KL, a=0.05, $n_\text{warm}$=1k)                                         & 97.98 (0.14)      & 2.16 (0.16)    & 0.97 (0.13)            & 59.56 (3.64)             \\
              & Bi-tempered ($t_1$=0.9, $t_2$=1.0)                                 & 98.09 (0.2)       & 2.04 (0.15)    & 1.07 (0.14)            & 55.82 (4.32)             \\
              & Mixup (a=0.2)                                  & 98.17 (0.04)       & 1.59 (0.07)    & 0.74 (0.04)             & 47.8 (2.53)             \\
              & Control                                                                   & 97.98 (0.13)      & 2.28 (0.13)    & 0.96 (0.07)            & 63.36 (2.55)             \\ \hline
              & $k$-NN LS (k=10, a=1, b=0.5)                              & 88.89 (0.14)      & 6.94 (0.18)    & 3.27 (0.15)            & 36.26 (1.09)             \\
              & Label Smoothing (a=0.8)                                                        & 88.46 (0.17)      & 7.2 (0.46)     & 3.32 (0.28)            & 36.63 (2.02)             \\
              & Anchor (a=0.9)                                                                    & 88.55 (0.14)      & 7.53 (0.45)    & 3.6 (0.23)             & 37.78 (2.29)             \\
Fashion  & $\ell_2$ Reg (a=0.5)                                                               & 88.52 (0.19)      & 7.86 (0.36)    & 3.59 (0.18)            & 40.38 (1.81)             \\
    MNIST          & $\ell_1$ Reg (a=0.1)                                                              & 86.88 (0.35)      & 8.24 (0.55)    & 3.88 (0.41)            & 36.81 (2.63)             \\
              & Co-distill (CE, a=0.5, $n_\text{warm}$=2k)                                          & 88.76 (0.21)      & 7.51 (0.39)    & 3.67 (0.3)             & 37.98 (1.71)             \\
              & Co-distill (KL, a=0.5, $n_\text{warm}$=2k)                                          & 88.85 (0.35)      & 7.83 (0.43)    & 3.68 (0.29)            & 40.59 (2.4)              \\
              & Bi-tempered ($t_1$=0.7, $t_2$=2)                                                 & 88.7 (0.29)       & 7.36 (0.47)    & 3.5 (0.19)             & 37.24 (3.04)             \\
              & Mixup (a=0.4)                                  & {\bf89.17 (0.10)}       & {\bf6.77 (0.29)}    & {\bf3.23 (0.15)}             & 35.97 (1.43)             \\
              & Control                                                                   & 88.95 (0.26)      & 9.13 (0.51)    & 4.42 (0.4)             & 46.99 (2.49)             \\ \hline
                      & $k$-NN LS (k=100, b=0.1, a=0.9) & {\bf 90.02 (0.11)} & {\bf 5.46 (0.32)}  & {\bf 2.97 (0.18)}      & 27.71 (1.74)       \\
                      & Label Smoothing (a=0.05)                             & 89.39 (0.29) & 6.77 (0.41)  & 3.81 (0.26)      & 31.67 (2.34)       \\
                      & Anchor (a=0.8)                                    & 89.87 (0.14) & 5.57 (0.28)  & 3.07 (0.21)      & 27.66 (1.38)       \\
CelebA        & $\ell_2$ Reg (a=0.01)                                   & 89.35 (0.16) & 6.85 (0.34)  & 3.92 (0.27)      & 31.62 (1.21)       \\
  Smiling       & $\ell_1$ Reg (a=0.5)                                    & 89.39 (0.26) & 6.71 (0.26)  & 3.61 (0.22)      & 32.48 (1.35)       \\
                      & Co-distill (CE, a=0.5, $n_\text{warm}$=1k)                & 89.59 (0.29) & 6.31 (0.23)  & 3.66 (0.3)       & 29.47 (1.47)       \\
                      & Co-distill (KL, a=0.5, $n_\text{warm}$=2k)                & 89.57 (0.22) & 6.1 (0.23)   & 3.34 (0.26)      & 29.66 (1.47)       \\
                      & Bi-tempered ($t_1$=0.9, $t_2$=2.)       & 89.88 (0.18) & 6.44 (0.31)  & 3.56 (0.19)      & 31.96 (1.96)       \\
              & Mixup (a=0.2)                                  & 89.71 (0.14)       & 6.15 (0.12)    & 3.51 (0.12)             & 29.37 (0.66)             \\
                      & Control                                         & 89.67 (0.19) & 7.3 (0.45)   & 4.06 (0.27)      & 35.34 (2.35)       \\ \hline
                      & $k$-NN LS (k=100, b=0.1, a=0.9) & {\bf 84.48 (0.21)} & {\bf 7.7 (0.29)}   & {\bf 4.64 (0.27)}      & 24.44 (0.98)       \\
                      & Label Smoothing (a=0.005)                            & 83.73 (0.17) & 8.68 (0.46)  & 5.2 (0.35)       & 26.61 (1.28)       \\
                      & Anchor (a=0.9)                                    & 84.48 (0.2)  & 7.97 (0.39)  & 4.77 (0.22)      & 25.44 (1.58)       \\
CelebA  & $\ell_2$ Reg (a=0.001)                                  & 83.6 (0.14)  & 9.06 (0.32)  & 5.41 (0.24)      & 27.66 (1.03)       \\
     High                  & $\ell_1$ Reg (a=0.01)                                   & 83.59 (0.26) & 8.43 (0.23)  & 4.93 (0.23)      & 26.14 (1.16)       \\
     Cheekbone                 & Co-distill (CE, a=0.5, $n_\text{warm}$=1k)                & 84.08 (0.21) & 8.96 (0.37)  & 5.33 (0.36)      & 28.11 (0.88)       \\
                      & Co-distill (KL, a=0.5, $n_\text{warm}$=1k)                & 84.31 (0.08) & 8.57 (0.16)  & 5.06 (0.13)      & 27.39 (0.47)       \\
                      & Bi-tempered ($t_1$=0.5, $t_2$=4)       & 83.92 (0.13) & 7.84 (0.32)  & 4.75 (0.21)      & 24.01 (1)          \\
              & Mixup (a=0.4)                                  & 84.53 (0.14)       & 7.92 (0.47)    & 4.69 (0.31)             & 25.53 (1.54)             \\
                      & Control                                         & 83.93 (0.56) & 10.18 (0.93) & 6.2 (0.89)       & 31.1 (2.22)        \\ \hline
                      & $k$-NN LS (k=500, a=0.8, b=0.9) & {\bf 96.69 (0.09)} & {\bf 1.04 (0.21)}  & {\bf 0.54 (0.14)}      & 15.81 (3.52)       \\
                      & Label Smoothing (a=0.8)                              & 96.63 (0.09) & 1.26 (0.26)  & 0.64 (0.17)      & 18.8 (3.42)        \\
                      & Anchor (a=0.9)                                    & 96.02 (0.25) & 2.33 (0.25)  & 1.08 (0.15)      & 31.58 (5.11)       \\
 Phishing          & $\ell_2$ Reg (a=0.5)                                    & 96.51 (0.12) & 1.35 (0.3)   & 0.7 (0.21)       & 19.37 (4)          \\
                      & $\ell_1$ Reg (a=0.5)                                    & 95.38 (0.18) & 1.48 (0.34)  & 0.83 (0.24)      & 14.95 (4.08)       \\
                      & Co-distill (CE, a=0.2, $n_\text{warm}$=2)                & 96.02 (0.19) & 1.45 (0.26)  & 0.83 (0.21)      & 16.72 (4.13)       \\
                      & Co-distill (KL, a=0.001, $n_\text{warm}$=1k)              & 95.94 (0.33) & 1.51 (0.2)   & 0.65 (0.18)      & 20.95 (6.14)       \\
                      & Bi-tempered ($t_1$=0.9, $t_2$=1.0)                     & 96.26 (0.37) & 2.32 (0.69)  & 1.23 (0.53)      & 30.19 (8.51)       \\
              & Mixup (a=0.1)                                  & 96.22 (0.23)       & 1.80 (0.33)    & 1.05 (0.28)             & 21.53 (4.25)             \\
                      & Control                                         & 96.3 (0.32)  & 2.25 (0.59)  & 1.21 (0.38)      & 29.05 (7.93)       \\ \hline
\end{tabular}
\caption{Results across all datasets and baselines under optimal hyperparameter tuning (settings shown). Note that we report the standard deviation of the runs instead of standard deviation of the mean (i.e. standard error) which is often reported instead. The former is higher than the latter by a factor of the square root of the number of trials (10).}

\label{table:main}
\end{table*}

\begin{figure*}[!t]
\begin{tabular}{cc}
  \includegraphics[width=0.45\textwidth]{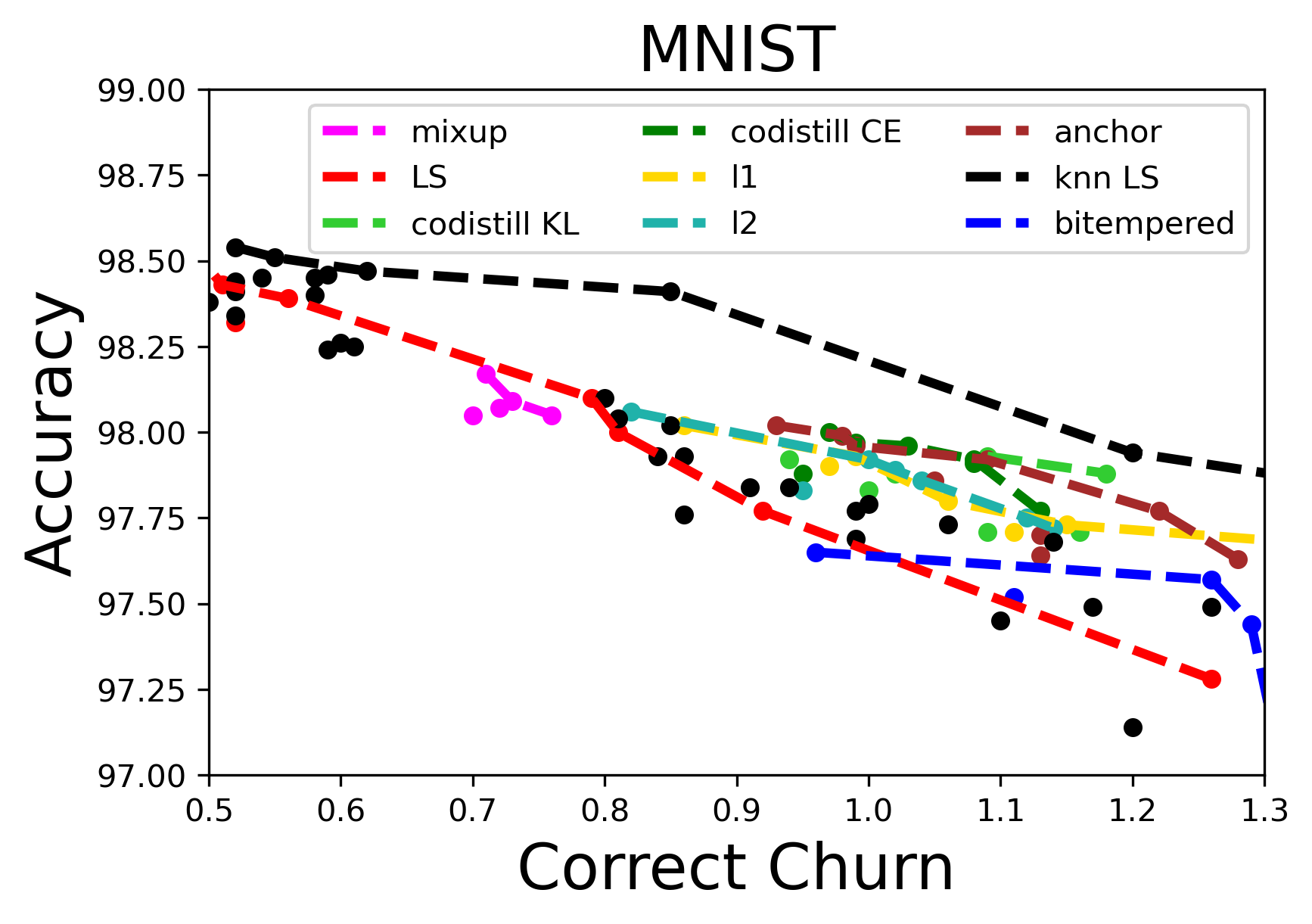} & 
  \includegraphics[width=0.45\textwidth]{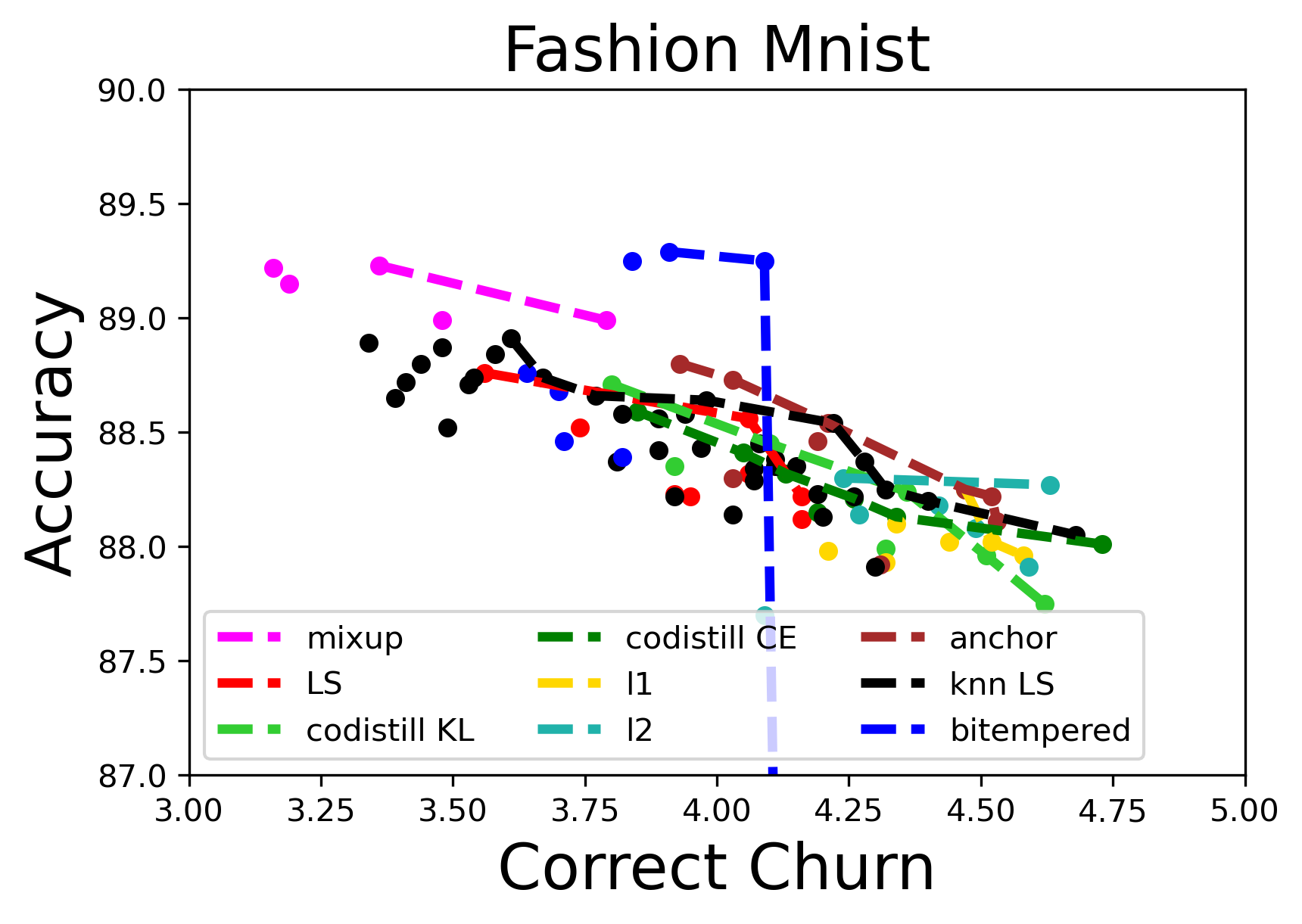} \\ 
  \includegraphics[width=0.45\textwidth]{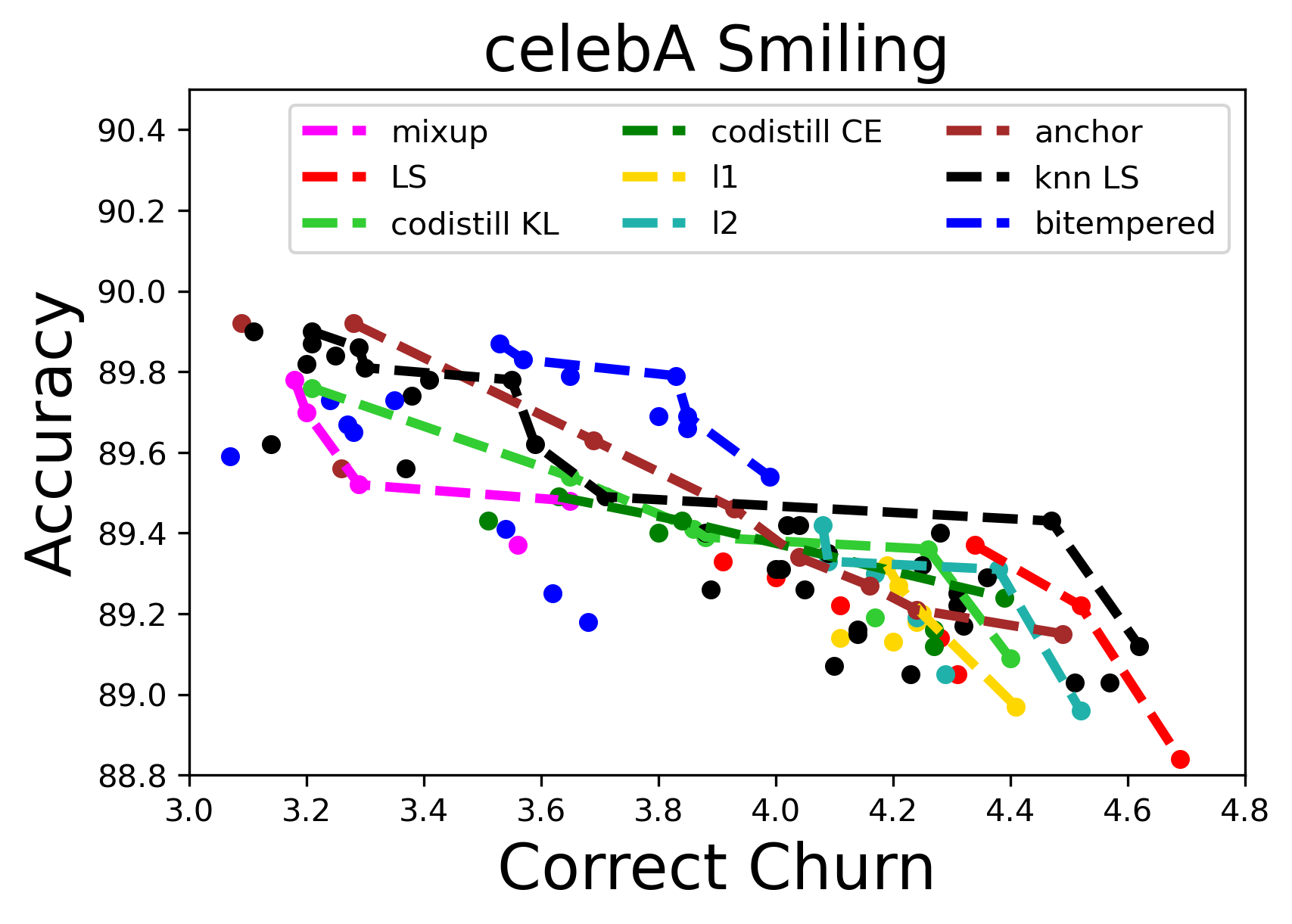} &
  \includegraphics[width=0.45\textwidth]{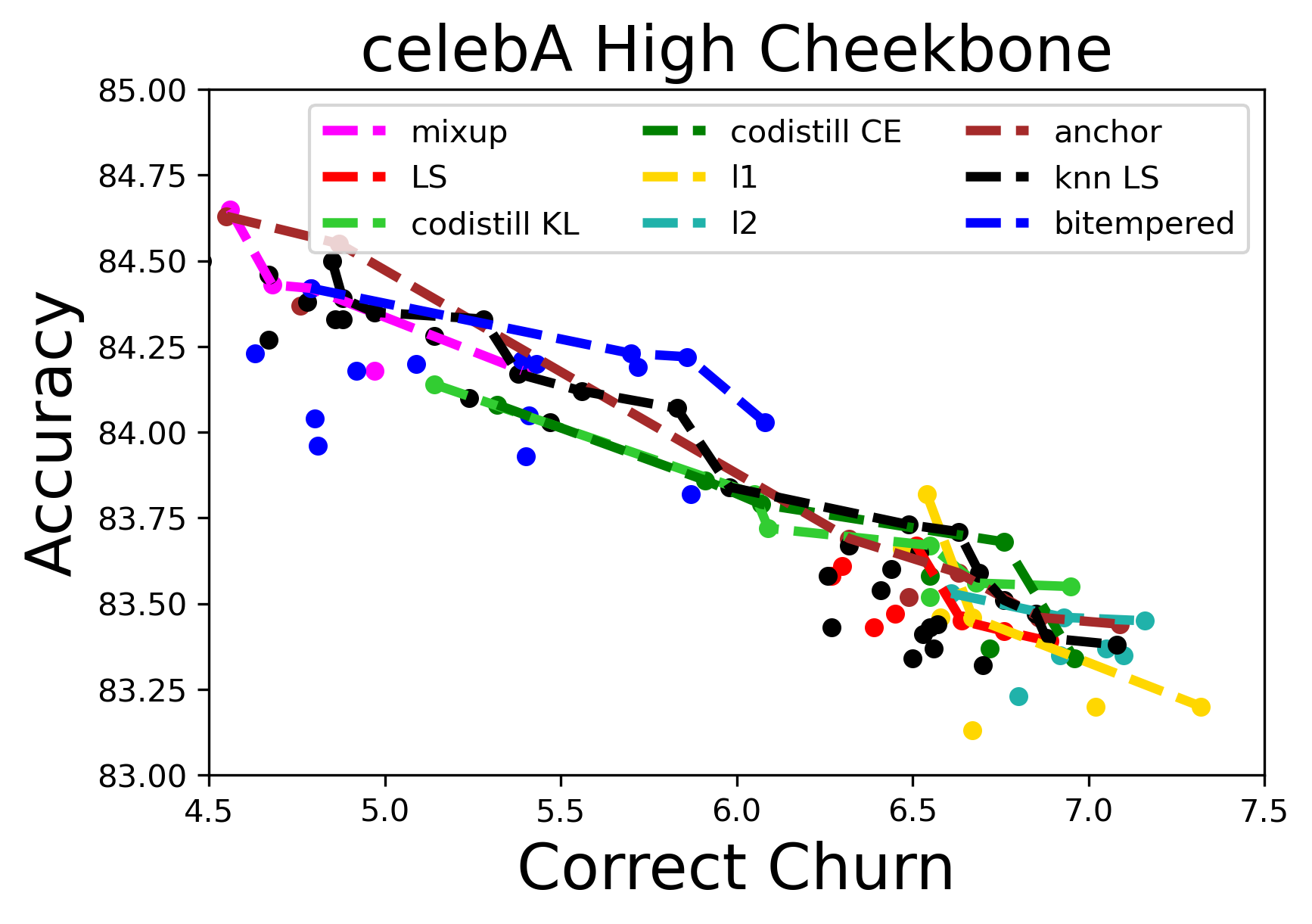} \\
  \includegraphics[width=0.45\textwidth]{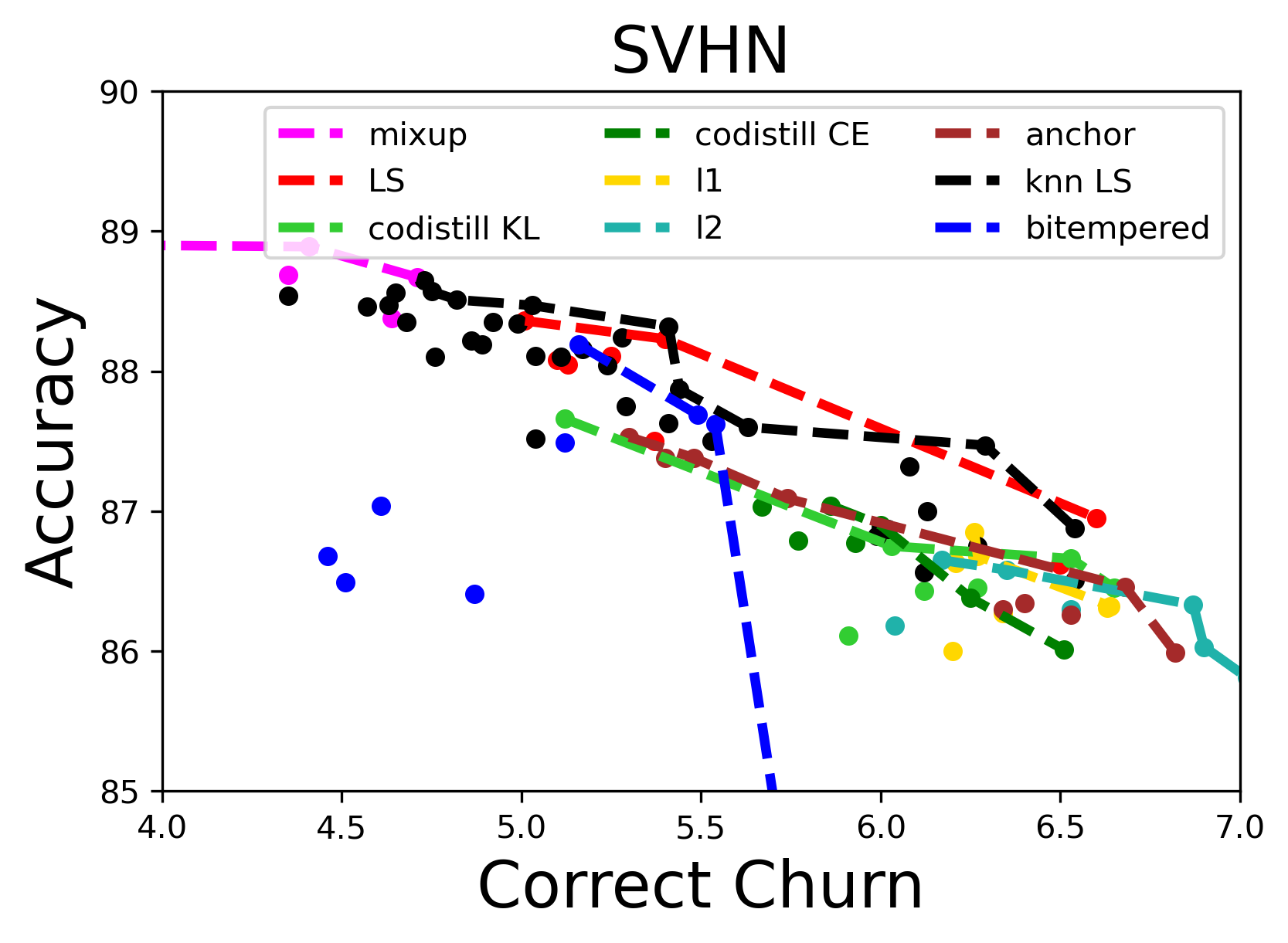} & 
  \includegraphics[width=0.45\textwidth]{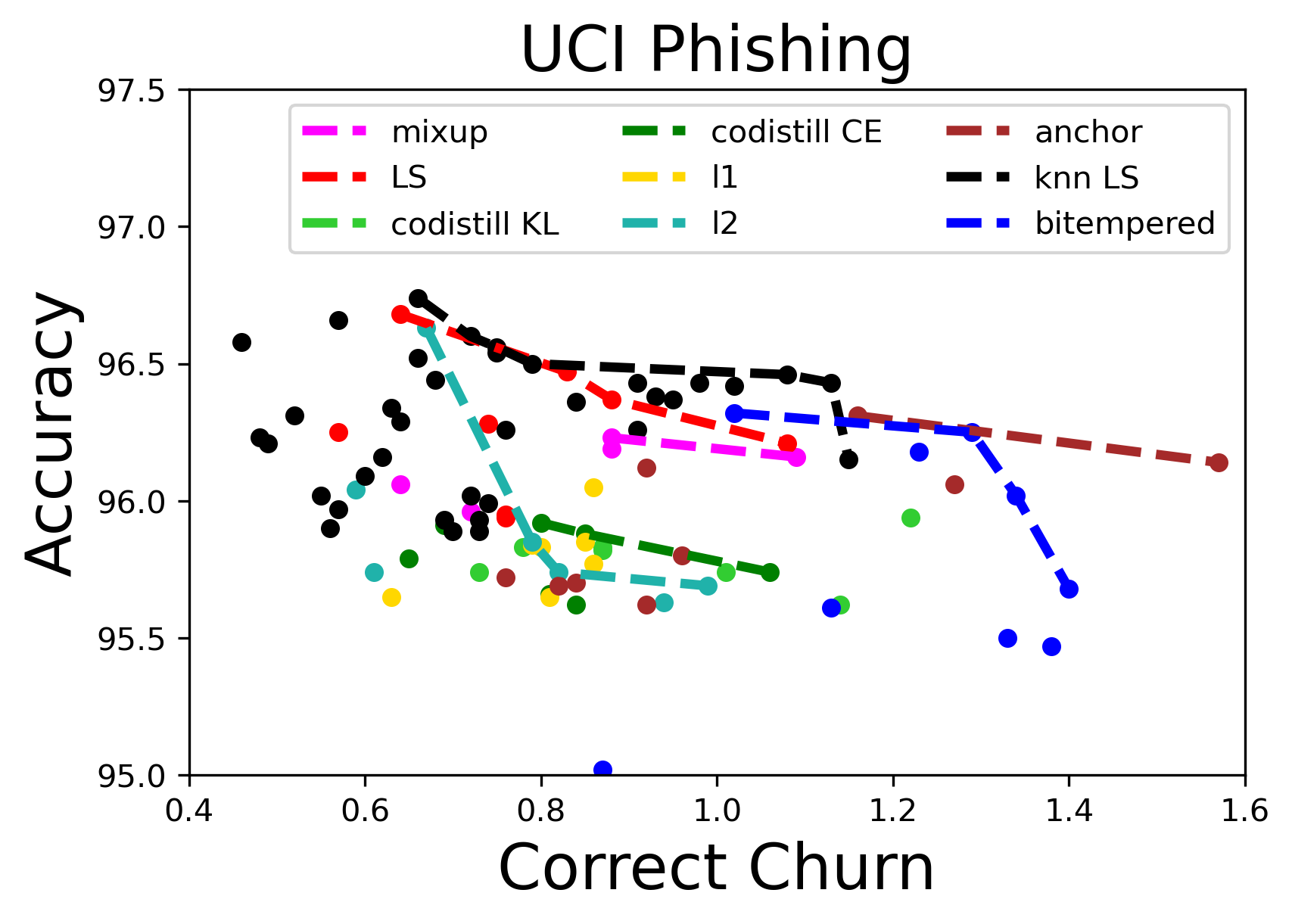} \\
\end{tabular}
\caption{\textbf{Accuracy vs. correct churn plots.} This figure plots accuracy against correct churn achieved by various hyperparameter settings of each method. The top-left corner (high accuracy, low churn) is most desirable. The dotted lines depict the Pareto frontier for each method, and we zoom in for a detailed view. We fix $k$-NN LS $k = 100$ and the number of codistillation warmup steps to $1000$. All other hyperparameters are sweeped as specified in the appendix. We see that $k$-NN label smoothing ranks reasonably high, mostly achieving low churn without sacrificing accuracy.}
\label{fig:pareto}
\end{figure*}

\subsection{Evaluation Metrics and Hyperparameter Tuning}
For each dataset, baseline and hyper-parameter setting, we run each method on the same train and test split exactly 5 times. We then report the average test accuracy as well as the test set churn averaged across every possible pair $(i, j)$ of runs (10 total pairs). To give a more complete picture of the sources of churn, we also slice the churn by the whether or not the test predictions of the first run in the pair were correct. Then, lowering the churn on the correct predictions is desirable (i.e. if the base model is correct, we clearly don't want the predictions to change), while churn reduction on incorrect predictions is less relevant (i.e. if the base model was incorrect, then it may be better for there to be higher churn-- however at the same time, some examples may be inherently difficult to classify or the label is such an outlier that we don't expect an optimal model to correctly classify, in which case lower churn may be desirable). This is why in the results for Table~\ref{table:main}, we bold the best performing baseline for churn on correct examples, but not for churn on incorrect examples.

In the results (Table~\ref{table:main}), for each dataset and baseline, we chose the optimal hyperparameter setting by first sorting by accuracy and choosing the setting with the highest accuracy, and if there were multiple settings with very close to the top accuracy (defined as within less than $0.1\%$ difference in test accuracy), then we chose the setting with the lowest churn among those settings with accuracy close to the top accuracy. There is often no principled way to trade-off the two sometimes competing objectives of accuracy and churn. \citet{cotter2019optimization} offer a heuristic to trade off the two objectives in a more balanced manner on the Pareto frontier. However in this case, biasing towards higher accuracy is most realistic because in practice, when given a choice between two models, it's usually best to go with the more accurate model. Fortunately, we will see that accuracy and churn are not necessarily competing objectives and our proposed method usually gives the best result for both simultaneously.

\subsection{Performance across hyperparameter settings}

In Figure~\ref{fig:hyperparameter}, we show the performance on SVHN w.r.t. the hyperparameters for both accuracy and churn. We fix two of the hyperparameters and show the results across tunings of the remaining hyperparameters. We do this for each of the three hyperparameters of our approach ($a$, $b$ and $k$). We see that larger $a$ corresponds to better performance, implying that less weight on the original labels leads to better results. We also see that across a wide range of $k$, the performance did not change much, which suggests that in practice, $k$ can be set to some default and not require tuning. Such stability in $k$ is desirable. Hence, the remaining hyperparameter $b$, which decides the trade-off between the locally adaptive vs global smoothing appears most essential. This further shows that our proposal of using locally adaptive label smoothing has a real effect on the results for both churn and accuracy.

\subsection{Results}

We see from Table~\ref{table:main} that mixup and our method, $k$-NN label smoothing, are consistently the most competitive; mixup outperforms on SVHN and Fashion MNIST while $k$-NN label smoothing outperforms on all the remaining datasets. Notably, both methods do well on accuracy and churn metrics simultaneously. Figure~\ref{fig:pareto} plots accuracy versus churn for different hyperparameter settings and highlights the Pareto frontier. We find that the $k$-NN label smoothing is often Pareto efficient.

Results for the ensemble baseline can be found in the Appendix. While we found ensembling to be remarkably effective, it does come with higher cost (more trainable parameters and higher inference cost), and so we discourage a direct comparison with other methods since an ensemble uses a different model class than a single model.


\section{Conclusion}

Modern DNN training is a noisy process: randomization arising from stochastic minibatches, weight initialization, data preprocessing techniques, and hardware can all lead to models with drastically different predictions on the same datapoints when using the same training procedure.

Reducing such prediction churn is important in practical problems as production ML models are constantly updated and improved on. Since offline metrics can usually only serve as proxies to the live metrics, comparing the models in A/B tests and live experiments oftentimes must involve manual labeling of the disagreements between the models, making it a costly procedure. Thus, controlling the amount of predictive churn can be crucial for more efficiently iterating and improving models in a production setting. 

Despite the practical importance of this problem, there has been little work done in the literature on this topic. We provide one of the first comprehensive analyses of reducing predictive churn arising from retraining the model on the same dataset and model architecture. We show that numerous methods used for other goals such as learning with noisy labels and improving model calibration serve as reasonable baselines for lowering prediction churn. We propose a new technique, locally adaptive label smoothing, that often outperforms the baselines across a range of datasets and model architectures.

Further study in this area is critical: the problem of predictive churn has received far too little treatment in the academic literature given its practical significance. Our technique may also help in the subfields that we drew many of our baselines from, including better calibrated DNNs and robustness to label noise, suggesting a bi-directional flow of ideas between the goal of reducing predictive churn and these subfields. This is a direction for future work.

\bibliography{references}

\begin{thebibliography}{45}
\providecommand{\natexlab}[1]{#1}
\providecommand{\url}[1]{\texttt{#1}}
\expandafter\ifx\csname urlstyle\endcsname\relax
  \providecommand{\doi}[1]{doi: #1}\else
  \providecommand{\doi}{doi: \begingroup \urlstyle{rm}\Url}\fi

\bibitem[Amid et~al.(2019)Amid, Warmuth, Anil, and Koren]{amid2019robust}
Amid, E., Warmuth, M.~K., Anil, R., and Koren, T.
\newblock Robust bi-tempered logistic loss based on bregman divergences.
\newblock In \emph{Advances in Neural Information Processing Systems}, pp.\
  14987--14996, 2019.

\bibitem[Anil et~al.(2018)Anil, Pereyra, Passos, Ormandi, Dahl, and
  Hinton]{anil2018large}
Anil, R., Pereyra, G., Passos, A., Ormandi, R., Dahl, G.~E., and Hinton, G.~E.
\newblock Large scale distributed neural network training through online
  distillation.
\newblock \emph{arXiv preprint arXiv:1804.03235}, 2018.

\bibitem[Bahri et~al.(2020)Bahri, Jiang, and Gupta]{bahri2020deep}
Bahri, D., Jiang, H., and Gupta, M.
\newblock Deep k-nn for noisy labels.
\newblock \emph{ICML}, 2020.

\bibitem[Beel et~al.(2013)Beel, Genzmehr, Langer, N{\"u}rnberger, and
  Gipp]{beel2013comparative}
Beel, J., Genzmehr, M., Langer, S., N{\"u}rnberger, A., and Gipp, B.
\newblock A comparative analysis of offline and online evaluations and
  discussion of research paper recommender system evaluation.
\newblock In \emph{Proceedings of the international workshop on reproducibility
  and replication in recommender systems evaluation}, pp.\  7--14, 2013.

\bibitem[Chaudhuri \& Dasgupta(2010)Chaudhuri and Dasgupta]{chaudhuri2010rates}
Chaudhuri, K. and Dasgupta, S.
\newblock Rates of convergence for the cluster tree.
\newblock In \emph{Advances in neural information processing systems}, pp.\
  343--351, 2010.

\bibitem[Chaudhuri \& Dasgupta(2014)Chaudhuri and Dasgupta]{chaudhuri2014rates}
Chaudhuri, K. and Dasgupta, S.
\newblock Rates of convergence for nearest neighbor classification.
\newblock In \emph{Advances in Neural Information Processing Systems}, pp.\
  3437--3445, 2014.

\bibitem[Cotter et~al.(2019)Cotter, Jiang, Gupta, Wang, Narayan, You, and
  Sridharan]{cotter2019optimization}
Cotter, A., Jiang, H., Gupta, M.~R., Wang, S., Narayan, T., You, S., and
  Sridharan, K.
\newblock Optimization with non-differentiable constraints with applications to
  fairness, recall, churn, and other goals.
\newblock \emph{Journal of Machine Learning Research}, 20\penalty0
  (172):\penalty0 1--59, 2019.

\bibitem[Cover(1968)]{cover1968rates}
Cover, T.~M.
\newblock Rates of convergence for nearest neighbor procedures.
\newblock In \emph{Proceedings of the Hawaii International Conference on
  Systems Sciences}, pp.\  413--415, 1968.

\bibitem[Deng(2015)]{deng2015objective}
Deng, A.
\newblock Objective bayesian two sample hypothesis testing for online
  controlled experiments.
\newblock In \emph{Proceedings of the 24th International Conference on World
  Wide Web}, pp.\  923--928, 2015.

\bibitem[Deng \& Shi(2016)Deng and Shi]{deng2016data}
Deng, A. and Shi, X.
\newblock Data-driven metric development for online controlled experiments:
  Seven lessons learned.
\newblock In \emph{Proceedings of the 22nd ACM SIGKDD International Conference
  on Knowledge Discovery and Data Mining}, pp.\  77--86, 2016.

\bibitem[Deng et~al.(2013)Deng, Xu, Kohavi, and Walker]{deng2013improving}
Deng, A., Xu, Y., Kohavi, R., and Walker, T.
\newblock Improving the sensitivity of online controlled experiments by
  utilizing pre-experiment data.
\newblock In \emph{Proceedings of the sixth ACM international conference on Web
  search and data mining}, pp.\  123--132, 2013.

\bibitem[Devroye et~al.(1994)Devroye, Gyorfi, Krzyzak, Lugosi,
  et~al.]{devroye1994strong}
Devroye, L., Gyorfi, L., Krzyzak, A., Lugosi, G., et~al.
\newblock On the strong universal consistency of nearest neighbor regression
  function estimates.
\newblock \emph{The Annals of Statistics}, 22\penalty0 (3):\penalty0
  1371--1385, 1994.

\bibitem[Dmitriev \& Wu(2016)Dmitriev and Wu]{dmitriev2016measuring}
Dmitriev, P. and Wu, X.
\newblock Measuring metrics.
\newblock In \emph{Proceedings of the 25th ACM international on conference on
  information and knowledge management}, pp.\  429--437, 2016.

\bibitem[Dua \& Graff(2017)Dua and Graff]{Dua:2019}
Dua, D. and Graff, C.
\newblock {UCI} machine learning repository, 2017.
\newblock URL \url{http://archive.ics.uci.edu/ml}.

\bibitem[Fard et~al.(2016)Fard, Cormier, Canini, and Gupta]{fard2016launch}
Fard, M.~M., Cormier, Q., Canini, K., and Gupta, M.
\newblock Launch and iterate: Reducing prediction churn.
\newblock In \emph{Advances in Neural Information Processing Systems}, pp.\
  3179--3187, 2016.

\bibitem[Fix \& Hodges~Jr(1951)Fix and Hodges~Jr]{fix1951discriminatory}
Fix, E. and Hodges~Jr, J.~L.
\newblock Discriminatory analysis-nonparametric discrimination: consistency
  properties.
\newblock Technical report, California Univ Berkeley, 1951.

\bibitem[Fort et~al.(2019)Fort, Hu, and Lakshminarayanan]{fort2019deep}
Fort, S., Hu, H., and Lakshminarayanan, B.
\newblock Deep ensembles: A loss landscape perspective.
\newblock \emph{arXiv preprint arXiv:1912.02757}, 2019.

\bibitem[Glorot \& Bengio(2010)Glorot and Bengio]{glorot2010understanding}
Glorot, X. and Bengio, Y.
\newblock Understanding the difficulty of training deep feedforward neural
  networks.
\newblock In \emph{Proceedings of the thirteenth international conference on
  artificial intelligence and statistics}, pp.\  249--256. JMLR Workshop and
  Conference Proceedings, 2010.

\bibitem[Goh et~al.(2016)Goh, Cotter, Gupta, and
  Friedlander]{goh2016satisfying}
Goh, G., Cotter, A., Gupta, M., and Friedlander, M.~P.
\newblock Satisfying real-world goals with dataset constraints.
\newblock In \emph{Advances in Neural Information Processing Systems}, pp.\
  2415--2423, 2016.

\bibitem[Han et~al.(2018)Han, Yao, Yu, Niu, Xu, Hu, Tsang, and
  Sugiyama]{han2018co}
Han, B., Yao, Q., Yu, X., Niu, G., Xu, M., Hu, W., Tsang, I., and Sugiyama, M.
\newblock Co-teaching: Robust training of deep neural networks with extremely
  noisy labels.
\newblock In \emph{Advances in neural information processing systems}, pp.\
  8527--8537, 2018.

\bibitem[Jiang(2019)]{jiang2019non}
Jiang, H.
\newblock Non-asymptotic uniform rates of consistency for k-nn regression.
\newblock In \emph{Proceedings of the AAAI Conference on Artificial
  Intelligence}, volume~33, pp.\  3999--4006, 2019.

\bibitem[Jiang et~al.(2018)Jiang, Kim, Guan, and Gupta]{TrustScores}
Jiang, H., Kim, B., Guan, M.~Y., and Gupta, M.~R.
\newblock To trust or not to trust a classifier.
\newblock In \emph{Advances in Neural Information Processing Systems
  (NeurIPS)}, 2018.

\bibitem[Lakshminarayanan et~al.(2017)Lakshminarayanan, Pritzel, and
  Blundell]{lakshminarayanan2017simple}
Lakshminarayanan, B., Pritzel, A., and Blundell, C.
\newblock Simple and scalable predictive uncertainty estimation using deep
  ensembles.
\newblock In \emph{Advances in neural information processing systems}, pp.\
  6402--6413, 2017.

\bibitem[LeCun et~al.(1998)LeCun, Bottou, Bengio, and
  Haffner]{lecun1998gradient}
LeCun, Y., Bottou, L., Bengio, Y., and Haffner, P.
\newblock Gradient-based learning applied to document recognition.
\newblock \emph{Proceedings of the IEEE}, 86\penalty0 (11):\penalty0
  2278--2324, 1998.

\bibitem[Liu et~al.(2018)Liu, Luo, Wang, and Tang]{liu2018large}
Liu, Z., Luo, P., Wang, X., and Tang, X.
\newblock Large-scale celebfaces attributes (celeba) dataset.
\newblock \emph{Retrieved August}, 15:\penalty0 2018, 2018.

\bibitem[Loshchilov \& Hutter(2015)Loshchilov and Hutter]{loshchilov2015online}
Loshchilov, I. and Hutter, F.
\newblock Online batch selection for faster training of neural networks.
\newblock \emph{arXiv preprint arXiv:1511.06343}, 2015.

\bibitem[Malach \& Shalev-Shwartz(2017)Malach and
  Shalev-Shwartz]{malach2017decoupling}
Malach, E. and Shalev-Shwartz, S.
\newblock Decoupling" when to update" from" how to update".
\newblock In \emph{Advances in Neural Information Processing Systems}, pp.\
  960--970, 2017.

\bibitem[M{\"u}ller et~al.(2019)M{\"u}ller, Kornblith, and
  Hinton]{muller2019does}
M{\"u}ller, R., Kornblith, S., and Hinton, G.~E.
\newblock When does label smoothing help?
\newblock In \emph{Advances in Neural Information Processing Systems}, pp.\
  4694--4703, 2019.

\bibitem[Papernot \& McDaniel(2018)Papernot and McDaniel]{papernot2018deep}
Papernot, N. and McDaniel, P.
\newblock Deep k-nearest neighbors: Towards confident, interpretable and robust
  deep learning.
\newblock \emph{arXiv preprint arXiv:1803.04765}, 2018.

\bibitem[Reeve \& Kaban(2019)Reeve and Kaban]{reeve2019fast}
Reeve, H.~W. and Kaban, A.
\newblock Fast rates for a {kNN} classifier robust to unknown asymmetric label
  noise.
\newblock \emph{arXiv preprint arXiv:1906.04542}, 2019.

\bibitem[Santurkar et~al.(2018)Santurkar, Tsipras, Ilyas, and
  Madry]{santurkar2018does}
Santurkar, S., Tsipras, D., Ilyas, A., and Madry, A.
\newblock How does batch normalization help optimization?
\newblock \emph{arXiv preprint arXiv:1805.11604}, 2018.

\bibitem[Scardapane \& Wang(2017)Scardapane and Wang]{scardapane2017randomness}
Scardapane, S. and Wang, D.
\newblock Randomness in neural networks: an overview.
\newblock \emph{Wiley Interdisciplinary Reviews: Data Mining and Knowledge
  Discovery}, 7\penalty0 (2):\penalty0 e1200, 2017.

\bibitem[Shorten \& Khoshgoftaar(2019)Shorten and
  Khoshgoftaar]{shorten2019survey}
Shorten, C. and Khoshgoftaar, T.~M.
\newblock A survey on image data augmentation for deep learning.
\newblock \emph{Journal of Big Data}, 6\penalty0 (1):\penalty0 1--48, 2019.

\bibitem[Singh et~al.(2009)Singh, Scott, Nowak, et~al.]{singh2009adaptive}
Singh, A., Scott, C., Nowak, R., et~al.
\newblock Adaptive {Hausdorff} estimation of density level sets.
\newblock \emph{The Annals of Statistics}, 37\penalty0 (5B):\penalty0
  2760--2782, 2009.

\bibitem[Song \& Chai(2018)Song and Chai]{song2018collaborative}
Song, G. and Chai, W.
\newblock Collaborative learning for deep neural networks.
\newblock In \emph{Advances in Neural Information Processing Systems}, pp.\
  1832--1841, 2018.

\bibitem[Stone(1977)]{stone1977consistent}
Stone, C.~J.
\newblock Consistent nonparametric regression.
\newblock \emph{The Annals of Statistics}, pp.\  595--620, 1977.

\bibitem[Szegedy et~al.(2016)Szegedy, Vanhoucke, Ioffe, Shlens, and
  Wojna]{szegedy2016rethinking}
Szegedy, C., Vanhoucke, V., Ioffe, S., Shlens, J., and Wojna, Z.
\newblock Rethinking the inception architecture for computer vision.
\newblock In \emph{Proceedings of the IEEE conference on computer vision and
  pattern recognition}, pp.\  2818--2826, 2016.

\bibitem[Theocharous et~al.(2015)Theocharous, Thomas, and
  Ghavamzadeh]{theocharous2015ad}
Theocharous, G., Thomas, P.~S., and Ghavamzadeh, M.
\newblock Ad recommendation systems for life-time value optimization.
\newblock In \emph{Proceedings of the 24th International Conference on World
  Wide Web}, pp.\  1305--1310, 2015.

\bibitem[Thulasidasan et~al.(2019)Thulasidasan, Bhattacharya, Bilmes,
  Chennupati, and Mohd-Yusof]{thulasidasan2019combating}
Thulasidasan, S., Bhattacharya, T., Bilmes, J., Chennupati, G., and Mohd-Yusof,
  J.
\newblock Combating label noise in deep learning using abstention.
\newblock \emph{arXiv preprint arXiv:1905.10964}, 2019.

\bibitem[Tsybakov et~al.(1997)]{tsybakov1997nonparametric}
Tsybakov, A.~B. et~al.
\newblock On nonparametric estimation of density level sets.
\newblock \emph{The Annals of Statistics}, 25\penalty0 (3):\penalty0 948--969,
  1997.

\bibitem[Turner \& Nowotny(2015)Turner and Nowotny]{turner2015estimating}
Turner, J.~P. and Nowotny, T.
\newblock Estimating numerical error in neural network simulations on graphics
  processing units.
\newblock \emph{BMC Neuroscience}, 16\penalty0 (198), 2015.

\bibitem[Zhang et~al.(2017)Zhang, Cisse, Dauphin, and
  Lopez-Paz]{zhang2017mixup}
Zhang, H., Cisse, M., Dauphin, Y.~N., and Lopez-Paz, D.
\newblock mixup: Beyond empirical risk minimization.
\newblock \emph{arXiv preprint arXiv:1710.09412}, 2017.

\bibitem[Zhang et~al.(2018)Zhang, Xiang, Hospedales, and Lu]{zhang2018deep}
Zhang, Y., Xiang, T., Hospedales, T.~M., and Lu, H.
\newblock Deep mutual learning.
\newblock In \emph{Proceedings of the IEEE Conference on Computer Vision and
  Pattern Recognition}, pp.\  4320--4328, 2018.

\bibitem[Zheng et~al.(2016)Zheng, Song, Leung, and
  Goodfellow]{zheng2016improving}
Zheng, S., Song, Y., Leung, T., and Goodfellow, I.
\newblock Improving the robustness of deep neural networks via stability
  training.
\newblock In \emph{Proceedings of the ieee conference on computer vision and
  pattern recognition}, pp.\  4480--4488, 2016.

\bibitem[Zhu et~al.(2018)Zhu, Gong, et~al.]{zhu2018knowledge}
Zhu, X., Gong, S., et~al.
\newblock Knowledge distillation by on-the-fly native ensemble.
\newblock In \emph{Advances in neural information processing systems}, pp.\
  7517--7527, 2018.

\end{thebibliography}
\bibliographystyle{icml2021}
{
\onecolumn

\appendix
\section{Proofs}

For the proofs, we make use of the following result from \citet{jiang2019non} which bounds the number of distinct $k$-NN sets on the sample across all $k$:

\begin{lemma}[Lemma 3 of \citet{jiang2019non}] \label{knncount}
Let $M$ be the number of distinct $k$-NN sets over $\mathcal{X}$, that is, $M := |\{ N_k(x) : x \in \mathcal{X} \}|$. 
Then $M \le D\cdot n^D$.
\end{lemma}

\begin{proof}[Proof of Theorem~\ref{theo:knn_bound}]
We have by triangle inequality and the smoothness condition in Assumption~\ref{assumption} that:
\begin{align*}
|\eta_k(x) - \eta(x)|
&\le  \left|\sum_{i=1}^n (\eta(x_i)- \eta(x)) \cdot \frac{1\left[ x_i \in N_k(x) \right]}{|N_k(x)|} \right|  + \left|\sum_{i=1}^n (y_i - \eta(x_i)) \cdot \frac{1\left[ x_i \in N_k(x) \right]}{|N_k(x)|} \right|\\
&\le  C_\alpha \cdot r_k(x)^{\alpha} + \left|\sum_{i=1}^n (y_i - \eta(x_i)) \cdot \frac{1\left[ x_i \in N_k(x) \right]}{|N_k(x)|} \right|.
\end{align*}
We now bound each of the two terms separately.

To bound $r_k(x)$, let $r = \left( \frac{2k}{\omega \cdot v_D \cdot n \cdot p_{X, 0}}\right)^{1/D}$. We have
$\mathcal{P}(B(x, r)) \ge \omega \inf_{x' \in B(x, r) \cap \mathcal{X}} p_X(x') \cdot v_D r^D \ge \omega p_{X, 0} v_D r^D = \frac{2k}{n}$, where $\mathcal{P}$ is the distribution function w.r.t. $p_X$.
By Lemma 7 of \citet{chaudhuri2010rates} and the condition on $k$, it follows that with probability $1 - \delta/2$, uniformly in $x \in \mathcal{X}$, 
$|B(x, r)\cap X| \ge k$, where $X$ is the sample of feature vectors. Hence, $r_k(x) < r$ for all $x \in \mathcal{X}$ uniformly with probability at least $1-\delta/2$.

Define $\xi_i := y_i - \eta(x_i)$. Then, we have that $-1 \le \xi_i \le 1$ and thus by Hoeffding's inequality, we have that $A_x := \sum_{i=1}^n (y_i - \eta(x_i)) \cdot \frac{1\left[ x_i \in N_k(x) \right]}{|N_k(x)|} = \sum_{i=1}^n \xi_i \cdot \frac{1\left[ x_i \in N_k(x) \right]}{|N_k(x)|}$ satisfies $P(|A_x| > t/k) \le 2 \exp \left( - t^2 / 2k\right)$. Then setting $t = \sqrt{2k}\cdot \sqrt{\log(4D/\delta) + D\log(n)}$ gives
\begin{align*}
    \mathbb{P}\left(|A_x| \ge \sqrt{\frac{2\log(4D/\delta) + 2D\log(n)}{k}}\right) \le \frac{\delta}{2 D \cdot n^D}.
\end{align*}
By Lemma 3 of \citet{jiang2019non}, the number of unique random variables $A_x$ across all $x \in \mathcal{X}$ is bounded by $D\cdot n^D$. Thus, by union bound, 
\begin{align*}
\mathbb{P}\left(\sup_{x\in X} |A_x| \ge \sqrt{\frac{2\log(4D/\delta) + 2D\log(n)}{k}}\right) \le \delta/2.
\end{align*}
The result follows.
\end{proof}

\begin{proof}[Proof of Theorem~\ref{theo:knn_linear}]
Let $X$ be the $n$ sampled feature vectors and let $x \in \mathcal{X}$. Define $k'(x) := |X \cap B(x, r_\beta(x))|$. We have:
\begin{align*}
    |\eta_k(x) - \widetilde{\eta}_\beta(x)| \le |\eta_{k'(x)}(x) - \eta_k(x)| + |\eta_{k'(x)}(x) - \widetilde{\eta}_\beta(x)|.
\end{align*}
We bound each of the two terms separately.
We have
\begin{align*}
    |k'(x) - k| = \left|\sum_{x \in X} 1[x \in B(x, r(x))] - \beta\cdot n\right|
\end{align*}
By Hoeffding's inequality we have
\begin{align*}
    \mathbb{P}(|k'(x) - k|  \ge t\cdot n) \le 2\exp(-2t^2n).
\end{align*}
Choosing $t = \sqrt{\frac{\log(4D/\delta) + D\log (n)}{2n}}$ gives us
\begin{align*}
    \mathbb{P}\left(|k'(x) - k| \ge \sqrt{\frac{n}{2}\cdot (\log(4D/\delta) + D\log (n))}\right) \le \frac{\delta}{2 D \cdot n^D}.
\end{align*}
By Lemma 3 of \citet{jiang2019non}, the number of unique sets of points consisting of balls intersected with the sample is bounded by $D\cdot n^D$ and thus by union bound, we have with probability at least $1 - \delta/2$:
\begin{align*}
    \sup_{x \in \mathcal{X}} |k'(x) - k| \le \sqrt{\frac{n}{2}\cdot (\log(4D/\delta) + D\log (n))}.
\end{align*}
We now have
\begin{align*}
|\eta_{k'(x)}(x) - \eta_k(x)| &\le \left|\frac{1}{k} - \frac{1}{k'(x)} \right| \min\left\{k, k'(x) \right\} + \min\left\{\frac{1}{k}, \frac{1}{k'(x)} \right\} |k - k'(x)|\\
&\le \frac{2}{k}\cdot |k - k'(x)| \le \sqrt{\frac{2\log(4D/\delta) + 2D\log (n)}{\beta\cdot n}}.
\end{align*}
where the first inequality follows by comparing the difference contributed by the shared neighbors among the $k$-NN and $k'(x)$-NN (first term on RHS) and contributed by the neighbors that are not shared (second term on RHS).

For the second term, define $A_x := X \cap B(x, r_\beta(x))$. For any $x'$ sampled from $B(x, r_\beta(x))$, we have that the expected label is $\widetilde{\eta}_\beta(x)$. Since $\eta_{k'(x)}(x)$ is the mean label among datapoints in $A_x$, then we have by Hoeffding's inequality that
\begin{align*}
    \mathbb{P}(|\eta_{k'}(x) - \widetilde{\eta}_\beta(x)| \ge k'(x)\cdot t) \le 2 \exp \left( - t^2 / 2k'\right).
\end{align*}
 Then setting $t = \sqrt{2k'}\cdot \sqrt{\log(4D/\delta) + D\log(n)}$ gives
\begin{align*}
    \mathbb{P}\left(|\eta_{k'(x)}(x) - \widetilde{\eta}_\beta(x)| \ge \sqrt{\frac{2\log(4D/\delta) + 2D\log(n)}{k'(x)}}\right) \le \frac{\delta}{2 D \cdot n^D}.
\end{align*}
By Lemma 3 of \citet{jiang2019non}, the number of unique sets $A_x$ across all $x \in \mathcal{X}$ is bounded by $D\cdot n^D$. Thus, by union bound, with probability at least $1 - \delta/2$L
\begin{align*}
|\eta_{k'(x)}(x) - \widetilde{\eta}_\beta(x)|  \le \sqrt{\frac{2\log(4D/\delta) + 2D\log(n)}{k'(x)}}.
\end{align*}
The result follows immediately for $n$ sufficiently large.
\end{proof}

\section{Ensemble Results}
\begin{table*}[!t]
\centering
\begin{tabular}{l|l|l|l|l|}
\hline
\textbf{Dataset (m=5)} & \textbf{Accuracy (\%)} & \textbf{Churn (\%)} & \textbf{Churn Correct} & \textbf{Churn Incorrect} \\ \hline
SVHN                           & 90.34 (0.31)           & 6.61 (0.19)         & 2.75 (0.28)            & 43.12 (1.49)             \\
MNIST                          & 98.5 (0.07)            & 0.94 (0.14)         & 0.44 (0.09)            & 33.74 (4.39)             \\
Fashion MNIST                  & 89.71 (0.12)           & 4.05 (0.14)         & 1.85 (0.05)            & 23.16 (1.29)             \\
CelebA Smiling                 & 90.56 (0.09)           & 3.35 (0.16)         & 1.82 (0.11)            & 17.95 (0.99)             \\
CelebA High Cheekbone          & 85.12 (0.16)           & 4.95 (0.2)          & 2.87 (0.1)             & 16.81 (1.24)             \\
Phishing                       & 96.11 (0.06)           & 0.54 (0.08)         & 0.29 (0.08)            & 6.77 (1.31)              \\ \hline
\end{tabular}
\caption{Ensemble results for all datasets. In all settings, the optimal $m$ (number of subnetworks) is 5. We see that compared to the other methods presented, ensembling does well in both predictive performance and in reducing churn. It does come at a cost, however: the model is effectively 5 times larger, making both training and inference more expensive.}
\label{table:ensemble}
\end{table*}
In Table~\ref{table:ensemble} we present the experimental results for the ensemble baseline. The method performs remarkably well, beating the proposed method and the other baselines on both accuracy and churn reduction across datasets. We do note, however, that ensembling does come at a cost which may prove prohibitive in many practical applications. Firstly, having $m$ times the number of trainable parameters, training time (if done sequentially) takes $m$ times as long, as does inference, since each subnetwork must be evaluated before aggregation.

\section{Ablation Study}
\begin{table}[!t]
\centering
\begin{tabular}{l|l|l|l|l}
\hline
\textbf{Fixed} & \textbf{Ablated} & \textbf{Accuracy (\%)} & \textbf{Churn (\%)} & \textbf{Churn Correct} \\ \hline
k = 10, a = 1      & b = 0              & 86.54 (0.67)           & 13.43 (0.58)        & 5.86 (0.57)            \\
               & b = 0.05           & 87.37 (0.38)           & 12.22 (0.31)        & 5.34 (0.31)            \\
               & b = 0.1            & 86.94 (0.65)           & 13.41 (0.39)        & 5.69 (0.57)            \\
               & b = 0.5            & 88.48 (0.52)           & 11.12 (0.5)         & 4.37 (0.35)            \\
               & b = 0.9            & 88.98 (0.33)           & 10.98 (0.28)        & 4.64 (0.29)            \\ \hline
k = 10, a = 0.5    & b = 0              & 84.44 (2.43)           & 15.85 (2.39)        & 6.73 (2.47)            \\
               & b = 0.05           & 79.64 (3.1)            & 22.02 (5.15)        & 10.28 (4.06)           \\
               & b = 0.1            & 79.88 (2.63)           & 21.09 (3.59)        & 10.25 (1.85)           \\
               & b = 0.5            & 84.44 (2.54)           & 14.33 (1.78)        & 6.52 (2.83)            \\
               & b = 0.9            & 81.06 (2.35)           & 20.53 (4.52)        & 8.68 (3.36)            \\ \hline
k = 10, b = 0.9    & a = 0.005          & 73.91 (3.01)           & 28.02 (5.66)        & 13.85 (4.82)           \\
               & a = 0.01           & 72.41 (4.86)           & 25.57 (5.78)        & 13.66 (7.01)           \\
               & a = 0.02           & 72.03 (1.79)           & 31.25 (7.25)        & 17.26 (6.56)           \\
               & a = 0.05           & 73.2 (3.33)            & 30.41 (6.2)         & 17.96 (6.04)           \\
               & a = 0.1            & 75.28 (1.98)           & 23.96 (4.76)        & 10.13 (4.25)           \\
               & a = 0.5            & 81.06 (2.35)           & 20.53 (4.52)        & 8.68 (3.36)            \\
               & a = 0.8            & 85.99 (0.73)           & 13.76 (0.75)        & 6 (0.83)               \\
               & a = 0.9            & 87.27 (0.41)           & 13.72 (0.41)        & 5.68 (0.32)            \\
               & a = 1.0            & 88.98 (0.33)           & 10.98 (0.28)        & 4.64 (0.29)            \\ \hline
k = 10, b = 0.5    & a = 0.005          & 71.45 (3.81)           & 21.14 (4.37)        & 11.5 (5.46)            \\
               & a = 0.01           & 74.73 (6.24)           & 25.24 (3.84)        & 8.28 (4.35)            \\
               & a = 0.02           & 73.59 (3.72)           & 29.47 (6.89)        & 17.52 (6.13)           \\
               & a = 0.05           & 74.17 (3.88)           & 20.26 (4.15)        & 5.79 (3.7)             \\
               & a = 0.1            & 72.43 (2.75)           & 25.77 (5.41)        & 13.42 (4.89)           \\
               & a = 0.5            & 84.44 (2.54)           & 14.33 (1.78)        & 6.52 (2.83)            \\
               & a = 0.8            & 87.26 (0.41)           & 11.76 (0.24)        & 4.62 (0.21)            \\
               & a = 0.9            & 86.85 (0.54)           & 12.54 (0.44)        & 5.25 (0.48)            \\
               & a = 1.0            & 88.48 (0.52)           & 11.12 (0.5)         & 4.37 (0.35)            \\ \hline
a = 1, b = 0.9     & k = 10             & 88.98 (0.33)           & 10.98 (0.28)        & 4.64 (0.29)            \\
               & k = 100            & 88.19 (0.19)           & 11.15 (0.23)        & 4.67 (0.17)            \\
               & k = 500            & 87.98 (0.62)           & 11.33 (0.35)        & 4.72 (0.55)            \\ \hline
\end{tabular}
\caption{Ablation on $k$-NN label smoothing's hyperparameters: $a$, $b$, and $k$ for the SVHN dataset.}
\label{table:ablation}
\end{table}

In Table~\ref{table:ablation}, we report SVHN results ablating $k$-NN label smoothing's hyperparameters: $k$, $a$, and $b$. We observe the following trends: with $a$ fixed to 1, both accuracy and churn improve with increasing $b$, and a similar relationship holds as $a$ increases with $b$ fixed to $0.9$. Lastly, both key metrics are stable with respect to $k$.

\section{Hyperparameter Search}
Our experiments involved performing a grid search over hyperparameters. We detail the search ranges per method below.
\paragraph{$k$-NN label smoothing.}
\begin{itemize}
    \item $k\in [5, 10, 100, 500]$
    \item $a \in [0.005, 0.01, 0.02, 0.05, 0.1, 0.5, 0.8, 0.9, 1.0]$
    \item $b \in [0, 0.05, 0.1, 0.5, 0.9]$
\end{itemize}
\paragraph{Anchor.}
\begin{itemize}
    \item $a \in [0.005, 0.01, 0.02, 0.05, 0.1, 0.5, 0.8, 0.9, 1.0]$
\end{itemize}
\paragraph{$\ell_1,\ell_2$ Regularization.}
\begin{itemize}
    \item $a \in [0.001, 0.01, 0.05, 0.1, 0.2, 0.5]$
\end{itemize}
\paragraph{Co-distill}
\begin{itemize}
    \item $a \in [0.001, 0.01, 0.05, 0.1, 0.2, 0.5]$
    \item $n_\text{warm} \in [1000, 2000]$
\end{itemize}
\paragraph{Bi-tempered}
\begin{itemize}
    \item $t_1 \in [0.3, 0.5, 0.7, 0.9]$
    \item $t_2 \in [1., 2., 3., 4.]$
    \item $n_\text{iters}$ always set to $5$.
\end{itemize}
\paragraph{Mixup}
\begin{itemize}
    \item $a \in [0.2, 0.3, 0.4, 0.5]$
\end{itemize}
\paragraph{Ensemble}
\begin{itemize}
    \item $m \in [3, 5]$
\end{itemize}
}
\end{document}